\documentclass{amsart}

\RequirePackage{graphicx}
\RequirePackage{caption}
\usepackage{multicol}

\usepackage[utf8]{inputenc}
\usepackage[T1]{fontenc} 
\usepackage{hyperref}   
\usepackage{url}         
\usepackage{booktabs}  
\usepackage{amsfonts}      
\usepackage{nicefrac}       
\usepackage{microtype}      
\usepackage{xcolor}     
\usepackage{amsthm}

\usepackage{amsmath}
\usepackage{amsfonts}
\usepackage{amssymb}
\usepackage{amsthm}
\usepackage{cuted}
\usepackage{enumitem}
\usepackage{url}

\usepackage{algorithm}
\usepackage{algpseudocode}

\newcommand{\pr}{\mbox{\sf P}}
\newcommand{\ex}{{\bf\sf E}}               

\newcommand{\R}{\mathbb{R}}

\newcommand{\bbS}{{\mathbb S}}   

\newcommand{\calm}{{\mathcal M}}
\newcommand{\calb}{{\mathcal B}}

\newcommand{\calp}{{\mathcal P}}
\newcommand{\calq}{{\mathcal Q}}

\newcommand{\calv}{{\mathcal V}}

\newcommand{\ra}{\rightarrow}           

\def\R{\mathbb{R}}
\def\ex{\mathbb{E}}

\newcommand{\Real}{\mathbb R}

\newcommand{\al}{\alpha}                
\newcommand{\g}{\lambda}                
   
\newcommand{\bmr}{\boldsymbol r}                              \newcommand{\bmu}{\boldsymbol\mu}                

\newtheorem{definition}{Definition}
\newtheorem{proposition}{Proposition}
\newtheorem{theorem}{Theorem}
\newtheorem{lemma}{Lemma}[section]
\newtheorem{remark}{Remark}[section]
\newtheorem{asm}{Assumption}
\newtheorem{cor}{Corollary}

\begin{document}

\title{On Representations of Mean-Field 
Variational Inference}

\author{Soumyadip Ghosh}
\author{Yingdong Lu}
\author{Tomasz Nowicki}
\address{IBM T.J. Watson Research Center}
\email{\{ghoshs,yingdong,tnowicki\}@us.ibm.com}

\author{Edith Zhang}
\address{Columbia University}
\email{ejz2120@columbia.edu}

\address{ } 



\begin{abstract}
The mean field variational inference (MFVI) formulation restricts the general Bayesian inference problem to the subspace of product measures.
We present a framework to analyse MFVI algorithms, which is inspired by a similar development for general variational Bayesian formulations. Our approach enables the MFVI problem to be represented in three different manners: a gradient flow on Wasserstein space, a system of Fokker-Planck-like equations and a diffusion process. Rigorous guarantees are established to show that a time-discretized implementation of the coordinate ascent variational inference algorithm in the product Wasserstein space of measures yields a gradient flow in the limit. A similar result is obtained for their associated densities, with the limit being given by a quasi-linear partial differential equation. A popular class of practical algorithms falls in this framework, which provides tools to establish convergence. We hope this framework could be used to guarantee convergence of algorithms in a variety of approaches, old and new,  to solve variational inference problems.
\end{abstract}

\maketitle

\section{Introduction}
\label{sec:intro}

Bayesian analysis posits a statistical model with observable variables $\pmb{x}\in\Real^n$ and unobserved latent variables $\theta \in \Real^d$ and seeks to infer a posterior distribution $p(\theta|\pmb{x})$ for the latent $\theta$ given a dataset of observations $\pmb{x} = ( x_1, \dots, x_n )$. The answer is provided, in the abstract, by Bayes' theorem: $p(\theta|\pmb{x}) = \pi(\theta) \,\, \pr(\pmb{x}|\theta)/Z$ where $\pr(\pmb{x}|\theta)$ represents the conditional probability of observations $\pmb{x}$ given $\theta$, $\pi(\theta)$ is a pre-specified prior distribution on $\theta$ and the normalizing constant $Z=\int \pi(\zeta) \pr(\pmb{x}|\zeta)\, d\zeta$ is the (unconditioned) probability of observing $\pmb{x}$. 
%
%
Computing the denominator $Z$ is often prohibitively expensive (it is a $\sharp P$-complete problem even in some special  cases, see, e.g.~\cite{mitzenmacher}), and so an exact computation of the desired posterior distribution $p$ directly from Bayes' rule is intractable. 
Various algorithms have been proposed to overcome this difficulty in practice. These include sampling algorithms such as Markov chain Monte Carlo (MCMC) methods~\cite{hastings70} that aim to estimate the true posterior $p$, but are challenged in practice by the possibility of long initialization periods that are discarded and the hardness of determining effective stopping criteria. \textit{Variational Inference (VI)} ~\cite{bleivi} algorithms on  the other hand can be efficiently implemented to quickly identify approximations of $p$ that are restricted to computationally advantageous forms. Each such VI approach comes with varying degrees of theoretical guarantees for convergence. 
In this article, we focus on the rigorous analysis of convergence of a subset called \textit{Mean Field VI (MFVI)}, a commonly implemented practical VI approach. 

The posterior distribution $p$ is trivially re-expressed as the  minimizer of the Kullback-Leibler (KL) divergence $D$ to itself, where $D(\xi\|\eta) := \ex_{\xi}[\log (d \xi/ d\eta)]$ for measures $\xi$ and $\eta$. Denoting $\pr(\theta, \pmb{x}) := \pi(\theta)\pr(\pmb{x}|\theta)$, we have 
\begin{align}
    p &= \arg\min_{\nu \in \calp(\R^d)} D(\nu\|p)    \label{opt:elbo}
    \\&= \arg\min_{\nu \in \calp(\R^d)}  \left\{ 
    \ex_\nu[\log\nu] - \ex_\nu [\log \pr(\pmb{x},\theta)] \right\} +\log Z.  
    \nonumber
\end{align}
Here, the set $\calp(\mathbb{R}^d)$ contains  absolutely continuous probability measures. The optimization problem~\eqref{opt:elbo} over the probability space is known as the \textit{Variational Bayes (VB)} form of Bayes' rule~\cite{bleivi}. Denote as $H(\nu) := -\ex_{\nu} [\log \nu]$ the entropy of the measure $\nu$, and $\Psi(\nu) := \ex_{\nu}[ - \log \pr(\pmb{x},\theta)] $ the expected negative log likelihood of the joint distribution $\pr(\pmb{x},\theta)$. Since $\log Z$ is a constant w.r.t. $\nu$, the VB~\ref{opt:elbo} minimizes the \textit{evidence lower bound (ELBO)}~\cite{bleivi} objective $J(\nu):=\Psi(\nu)-H(\nu) $. Equivalently, it maximizes $-J(\nu)$, balancing a high log likelihood $\Psi(\nu)$  under $\nu$ with a regularization term that desires a high entropy solution $\nu$. ~\cite{trillos20alonso} provide equivalent functional representations of the objective of~\eqref{opt:elbo} that arise from other perspectives. 


Existence, uniqueness and convergence results for VB can be obtained from representations of~\eqref{opt:elbo} constructed by exploiting  intriguing connections between Bayesian inference, differential equations and diffusion processes. \cite{jko98} provided a seminal result that the gradient flow in Wasserstein space (the metric space $\calp(\Real^d)$ of probability measures endowed with $2$-Wasserstein distance $W_2$) of an objective function like~\eqref{opt:elbo} can be equivalently expressed as the solution to a Fokker-Planck (FPE) equation, which is a parabolic partial differential equation (PDE) on densities as $L_1$ functions. 
These key connections allow Bayes' rule to be expressed as minimum of various related functionals on different metric spaces: it can be viewed as the stationary solution of a gradient flow of $J$ in the space $W_2$, as the stationary solution to an FPE in the $L_1$ space of density functions, and also corresponds to the stationary distribution of a diffusion process. These equivalent relationships have been depicted in Fig.~\ref{fig:perspective}; see~\cite{trillos20alonso} for further details. 

Solution procedures for the several equivalent optimization representations to obtain the posterior $p$ that are shown in Fig.~\ref{fig:perspective} are in practice hard to  
implement since each still requires computationally difficult operations in functional and probability spaces. In practice, the  VB problem~\eqref{opt:elbo} is approximated by \textit{Variational Inference} procedures that replace the general set $ \calp(\R^d)$ with a constrained subset of feasible probability measures $\calq \subset \calp$ where measures in $\calq$ possess structural properties that allow for practical and efficient implementation of the optimization. The solution thus obtained is an approximation of $p$, and will coincide only if $p\in\calq$. A common choice is the mean field VI \cite{bleivi} where $\calq$ is taken to be the mean field family $\calq(\R^d) := \prod_{i=1}^d \calp(\R)$ where the components of $\theta$ are independent of each other. The MFVI approximation of $p$ is then obtained by solving the optimization problem~\eqref{opt:elbo} over the restricted feasible set $ \nu \in \calq$.

\textbf{Contributions:} Our main focus is to derive multiple representations for the MFVI formulation similar to those displayed in Fig.~\ref{fig:perspective}. Our analogous representations are recounted concisely in Fig.~\ref{fig:perspective_MFVI}. 
Specifically: 
\begin{itemize}
\item
Broadly following the alternative views available for Bayesian inference, we describe three different representations of the MFVI algorithm. The first views the mean-field approximation of the posterior as the gradient flow of a joint set of functionals, the second as a solution to a system of quasilinear partial differential equations and the last as a diffusion process that is the stationary distribution of a system of stochastic differential equations. 
\item Theorem~\ref{pro:gradientflow} 
shows that a discrete process induced by the candidate solutions of a coordinate-wise algorithm (see Sec~\ref{sec:mfvi_formulation}) converges to an equivalent gradient flow defined on the product Wasserstein space of  measures when a certain step size parameter is shrunk to zero. This is to the best of our knowledge the first gradient flow representation of the general MFVI algorithm, and it depends on extensions of some basic concepts of gradient flows to product Wasserstein space, which are presented in Sec. \ref{sec:gradient_flow} and Sec. \ref{sec:gradient_flow_appendix}. 
\item
We also demonstrate that the corresponding density functions converge to the solution of a second order quasilinear evolutional (parabolic) equation in Corollary~\ref{cor:density_convergence}.  Additionally, in Theorem~\ref{thm:pde_uniqueness}, we extend our analysis to present new results of independent interest on existence and uniqueness of solutions to families of quasilinear evolutional equations that satisfy similar conditions. 
\item The quasilinear evolutional equation leads to the probabilistic representation of the MFVI by connecting its solution to the density of a stochastic process that is the solution to a corresponding stochastic differential equation (SDE) of Mckean-Vlasov type. 
\end{itemize}


The three representations presented in  this article open the possibility of multiple new algorithmic approaches to obtaining the approximation to $p$ in the space $\calq$, and also provides tools to study the convergence properties of these algorithms. While a detailed development is out of scope here, we briefly summarize some possibilities. The MFVI formulation~\eqref{opt:elbo} can be solved using a system of SGD-like iterations produced by Euler-discretization formulations~\eqref{opt:vi-euler}, each of which can be solved explicitly for further restrictions of the marginals measures to parametric families such as Gaussian, mixed-Gaussian etc. Alternately, non-parametric particle-based heuristics can be used to approximate the solution to the general SGD steps~\eqref{opt:vi-euler}. The SDE representation on  the other hand suggests that the posterior be approximated by estimating the stationary process of the SDE by exploiting techniques from the vast literature on SDEs. A particle filter based approach can for example be constructed using a system of MCMCs with dynamics arising from the components of the SDE.


\textbf{Prior Work:} Convergence analysis of the MFVI approximation to the VB problem is relatively less well established. \cite{wang19blei} provide consistency results for MFVI procedures by  establishing that point estimates of the latent variables $\theta$ (such as expectations of functions of $\theta$) constructed using MFVI estimates of the posterior converge to the true value asymptotically as the size $n$ of $\pmb{x}$ grows under the assumption that the true latent variable takes a definite value. A recent analysis by~\cite{FBach_VI} presents a convergence analysis of VI where the set $\calq$ are further constrained to be (mixtures of) Gaussian distributions, thus operating in the sub-manifold of $\calp(\Real^d)$ known as the Bures-Wasserstein manifold. 
Their methodology closely follows the standard VB analysis outlined in Fig.~\ref{fig:perspective} restricted to this manifold. In particular, the formulation~\eqref{opt:elbo} in this case leads to a simplified FPE equation and associated diffusion process, unlike our case which requires the development and analysis of a system of quasi-linear PDEs and associated stochastic processes.    


\textbf{Organization:} The rest of the paper will be organized as follows: in Sec.~\ref{sec:bayes}, we provide precise definitions of various key representations of Bayesian inference as illustrated in Fig.~\ref{fig:perspective}; Sec.~\ref{sec:mfvi_formulation} defines the optimization formulation of the MFVI problem, including an Euler discretization scheme which forms the basis of all the development that follows; 
in Sec.~\ref{sec:gf_MFVI}, we present our results on the convergence of the discrete scheme to gradient flow; in Sec.~\ref{sec:pde}, we define the equivalent quasilinear parabolic equation in the $L_2$ space of densities and discuss its well-posedness, as well as the probabilistic representation in the form of a Mckean-Vlasov stochastic differential equation.

\section{Representations of the Bayesian Posterior}
\label{sec:bayes}

The variational formulation ~\eqref{opt:elbo} of Bayesian inference enables a fruitful exploration of connection between the algorithms such as gradient descent and gradient flows, as well as its implications. This section presents a brief overview of three equivalent representations of the VB formulation that yield different 
characterisations 
of the posterior distribution, as summarized in 
Figure \ref{fig:perspective}, each of which lead to potential algorithmic approaches to approximate it. In Sec. \ref{sec:gf_MFVI} \&~\ref{sec:pde}, a similar set of relationship will be established for MFVI. For that purpose, we will provide precise definitions and descriptions of these characterizations in this section.  

\begin{figure}
	\centering
	\includegraphics[width=0.95\linewidth]
	{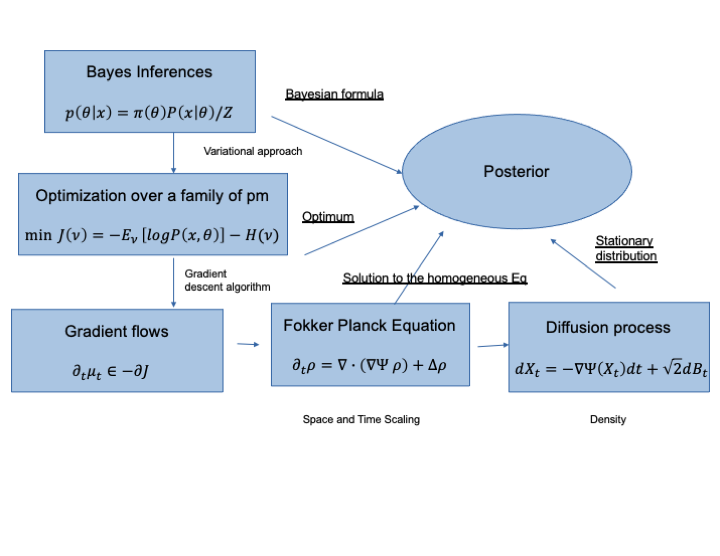}
	\vskip -0.5in
	\caption{Representations of Bayesian inference} 
	\label{fig:perspective}
\end{figure}

In the classic Euclidean space setting, a curve $x(t)$ in $\Real^d$ is called the gradient flow for some function $E:\Real^d \ra \Real$ if it solves the following equation
\begin{equation}
    \partial_t x(t) = -\nabla E(x(t)).
    \label{eq: gradflow}
\end{equation}
To extend this concept to general metric space (Wasserstein spaces included), consider a general (energy) functional $E:X\ra \Real$ defined on a metric space $(X,d)$, its gradient flow $x(t): \Real_+ \ra X$ solves the following {\it energy dissipation equation}, for $t>0$,
\begin{align}
E(x_0) \,=\, E(x_t) + \frac12 \int_0^t |{\dot x}((r)|^2 dr 
+ \frac12 \int_0^t |\nabla E(x(r))|^2 dr,\label{eqn:Energy_Dissipation}
\end{align}
where 
${\dot x}(t) :=\lim_{s\ra t} \frac{d(x(t), x(s))}{|s-t|}$, $|\nabla E|(x) :=\lim\sup_{y\ra x} \frac{(E(y)-E(x))^+}{d(x,y)}.$
While \eqref{eqn:Energy_Dissipation} provides a concise form of the essence of the gradient flow (see, e.g.~\cite{trillos20alonso}), more extensive characterization and discussion on the topic can be found in~\cite{ambrosio2006gradient}. 

The gradient flow of the VB problem~\eqref{opt:elbo},
as summarized in \cite{trillos20alonso}, 
is similarly derived by treating $J(\nu)$ as a functional on the Wasserstein space of probability measures $\nu$. Thus, the posterior distribution which minimizes the function $J(\nu)$ can be viewed naturally as the limit of the gradient flow on the Wasserstein space defined by $J(\nu)$.
Specifically, $J(\nu)$ is defined over the metric space of measures $(\calp^2(\mathbb{R}^d), W_2)$, where $\calp^2(\mathbb{R}^d)$ is the space of square-integrable probability measures on $\Real^d$:
\begin{align*}
\calp^2(\mathbb{R}^d)=\left\{\nu \in \calp(\mathbb{R}^d): \exists x_0, \int_{\Real^d} ||x-x_0||_2^2 d\nu(x)<\infty  \right\}.
\end{align*}
 The Wasserstein distance $W_2$ is defined, for $\nu_1, \nu_2 \in \calp(\mathbb{R}^d)$, via its squared value
\begin{align*}
W_2(\nu_1, \nu_2)^2 :=\inf_\al \int_{\Real^d \times \Real^d} ||x-y||_2^2 d\al(x,y). 
\end{align*}
with the infimum taken over all joint measures $\al \in \calp(\Real^d \times \Real^d)$ with marginals $\nu_1$ and $\nu_2$ on the first and second factors. 

The uniqueness of the gradient flow as well as its rate of convergence is established by the convexity properties of the functional $J(\nu)$. 
The rate of convergence is related, in~\cite{trillos20alonso}, to the $\g$-geodesic convexity (see Definition~\ref{defn:convexity} below)
of $J(\nu)$ for some $\g\in \Real$, and in
Prop.~3.1 they identify the necessary and sufficient conditions for $J$ to satisfy this convexity condition.
In the case of our Bayesian inference, where the underlying space is the finite dimensional Euclidean space and the reference measure is the Lebesgue measure, the condition is reduced to that the integrand of $\Psi$ is $\g$-convex (see Defn.~\ref{defn:l-convex-fn}). 

A straightforward procedure to access the gradient flow to the posterior is to follow a discretization over the parameter $t$. For instance, the gradient flow $x(t)$ in the Euclidean space can be discretized by the 
Euler scheme with step size $h>0$
\begin{equation}
    x^{k+1} = \arg\min_x \left\{\frac12 ||x - x^k||^2 + hE(x)\right\},
    \label{eq: euclidean euler}
\end{equation}
where the solution 
$x^{k+1} = x^k - h \nabla E(x^k)$ is a discretization of~\eqref{eq: gradflow}.
Note that the objective of the iterative Euler discretization scheme \eqref{eq: euclidean euler} does not include a gradient term. An analogous scheme to access the VB gradient flow thus only requires evaluating $J(\nu)$ and a $W_2$ distance term:
\begin{equation}
    \nu^{k+1} = \arg\min_{\nu} \left\{\frac12 W_2(\nu, \nu^k)^2 + h J(\nu)\right\}.
    \label{eq: metric euler}
\end{equation} 

\cite{jko98} establish a connection between gradient flows under $J(\nu)$ in $\calp_2$ and the Fokker-Planck partial differential equation satisfied by the densities of $\nu$. We denote by $\rho(\theta) : \Real^d \ra \Real_{+}$ a non-negative square integrable function as the Radon-Nikodym density of measure $\nu$ w.r.t. the Lebesgue reference measure. Re-write the functional $J$ as $J(\rho) = \Psi(\rho) - H(\rho)$, where $H(\rho) = -\int \rho\,\log \rho\, d\theta$ and $\Psi(\rho) =  - \int \log \pr(\pmb{x},\theta) \,\rho\, d\theta$. \cite{jko98} study the  Fokker-Planck equation over a collection $\{\rho(t)\}_{t\in\Real_{+}}$ 
\begin{align}
    (\text{FPE}): 
    \begin{cases}
    \partial_t \rho (t) = \nabla \cdot (\nabla \Psi(\rho(t)) ) + \Delta \rho(t)\\
    \rho(0) = \rho^0
    \end{cases}, \label{eq:fpe}
\end{align}
where the derivative $\nabla$ and the Laplacian $\Delta$ are over $\theta$.
(FPE) is a partial differential equation that governs the evolution (w.r.t $t$) of the probability density associated with a particle undergoing diffusion in a potential field $\Psi$.
\cite{jko98} 
state that solution of (FPE) is equivalent to the gradient flow of $J(\nu)$ in $(\calp_2(\Real^d),W_2)$. This is established by analysing the densities $\rho^{k}$ associated with the solutions $\nu^k$ of the gradient flow discretizing iterations in~\eqref{eq: metric euler}. Define the interpolation $\rho_{h} : (0,\infty) \times \R^d \rightarrow \Real_+^d$ by $\rho_{h}(t) = \rho_{h}^k$ for $t \in [kh, (k+1)h)$ and $k \geq 1, i=1,2,\ldots, d$. 
Their main theorem states: 
\vskip 0.2cm

\noindent
{\bf Theorem }[Theorem 5.1 from~\cite{jko98}]
\emph{Suppose that $\rho_h(t)$ is the continuous interpolation of solutions $\rho^k$ to \eqref{eq: metric euler} for a fixed $h$. Then as $h\ra 0$, $\rho_h(t) \ra \rho(t)$ weakly in $L^1(\Real^d)$ where $\rho(t)$ is the unique solution of \eqref{eq:fpe}. }
\vskip 0.2cm


The density satisfying the partial differential equation (FPE) also coincides with the density of a diffusion process with generator derived from the differential operator of (FPE). This further means that the homogeneous solution to  (FPE) is the same as the stationary distribution of the diffusion process, all being equal to (the density of) the desired posterior distribution. Indeed, the diffusion interpretation gives rise to the possibility of using MCMC techniques to approximate the posterior via appropriate discretizations of the diffusion~\cite{trillos20alonso}. A more detailed description of this relationship can be found in \ref{sec:diffusion} of the Appendix.






\section{MFVI Formulation} \label{sec:mfvi_formulation}
The MFVI approximation of the posterior solves:
\begin{align}
\label{eqn:MFVI_formulation}
\min_{\nu\in \calq(\R^d)=\prod_{i=1}^d \calp^2_i(\R)} J(\nu),
\end{align}
with $J(\nu)$ treated as a functional of the product space of measures $ \calq(\R^d)=\prod_{i=1}^d \calp^2_i(\R)$.
This is encoded via the constraint $\nu(\theta) = \prod_{i=1}^d \nu_i(\theta_i)$. Thus, $J(\nu)$ in~\eqref{opt:elbo} takes the form: 

\begin{align}
J(\nu)
=& -\int_{\R^d} \log \pr(x,\theta) \prod_{i=1}^d\nu_i(\theta_i)d\theta_i \nonumber  + \int_{\R^d} \sum_{i=1}^d \log \nu_i(\theta_i) \prod_{j=1}^d\nu_j(\theta_j)d\theta_j
 \nonumber\\
= & -\int_{\R^d} \log \pr(\pmb{x},\theta) \prod_{i=1}^d\nu_i(\theta_i)d\theta_i - \sum_{i=1}^d H(\nu_i) \label{eq: j star}
\end{align}
where 
$H(\nu_i) = -\int_\R \nu_i \log \nu_i d\theta_i$ is the entropy of the $i$-th component. 
%
Introduce 
\begin{align*}
\nu_{-i} := & \prod_{j\neq i} \nu_j,
\\
\Psi_i(\theta;\nu_{-i}) := & -\int_{\R^{d-1}} \log \pr(\pmb{x},\theta) \prod_{j\neq i} \nu_j(\theta_j)d\theta_j
\\= & \;\;\ex_{-i} [-\log \pr(\pmb{x},\theta)].
\end{align*}
This notation now lets us denote  
\begin{align*}J_i(\nu_i\,;\,\nu_{-i}) := 
\int \Psi_i(\theta;\nu_{-i}) \nu_i d\theta_i - H(\nu_i)
\label{eqn:def J-i}
\end{align*}
as the objective function restricted to $\nu_i$ with the other components $\nu_{-i}$ held fixed. It is apparent that, given $\nu_{-i}$, $J_i(\nu_i\,;\,\nu_{-i})$ is a functional on  $\calp_i^2(\R)$ for each $i=1,\ldots, d$. Furthermore, the following result is obtained in a straightforward manner.
\begin{proposition}
\label{pro:vi_system}
Suppose that probability measure $\nu^{\ast}\in\calq=\prod_{i=1}^d \nu_i^*\in \calq$ with $\nu^*_i\in \calp_i^2(\R), i=1, \ldots, d$
satisfying the following system of implicit equations 
\begin{equation}\label{eq:vi_system}
\nu^*_i = \arg\min_{\nu_i\in\calp(\Real)}  J_i(\nu_i\,;\,\nu^*_{-i}),\quad \forall i=1,\ldots,d.
\end{equation}
Then, $\nu^{\ast}$ is a solution to the optimization problem \eqref{eqn:MFVI_formulation}. 
\end{proposition}

Proposition \ref{pro:vi_system} motivates the common practice of solving for the MFVI solution $\nu^{\ast}$ by iteratively solving each of the $\nu^{\ast}_i$ holding the others fixed, and cycling through the components until convergence. This is called the \emph{coordinate ascent variational inference (CAVI)} framework in the VI literature ~\cite{bishop06} when the optimization problem is expressed as that of maximizing $-J(\nu)$.
So, CAVI generates solutions from the sequence of problems:
\begin{equation}\label{eq:vi_system-peri}
\nu^{k}_i = \arg\min_{\nu_i\in\calp(\Real)}  J_i(\nu_i\,;\,\nu^{k-1}_{-i}),\quad \forall i=1,\ldots,d,
\end{equation}
with $\nu^{k-1}_{-i}:= \prod_{j=1}^{i-1} \nu^{k}_j \prod_{j=i+1}^{d} \nu^{k-1}_j$ for all $i=1,\ldots, d$. 
With mild regularity assumptions, the limit as $k\nearrow\infty$ of the product measure converges to a solution to~\eqref{eq:vi_system}. 

While the CAVI framework assumes that~\eqref{eq:vi_system-peri} can be solved to optimality within each iteration, this can be inexpensively implemented 
only in specific instances where for example the feasible set $\calp(\Real)$ for each component is further restricted to have an advantageous parametric form; see~\cite{bleivi}. In the general space of measures $\calp(\Real)$, solving each optimization problem  \eqref{eq:vi_system-peri} to optimality poses significant computational challenges. The same fundamental issue arises here as that in implementing~\eqref{opt:elbo} in that the gradient of $J(\nu)$ w.r.t. $\nu$ needs to be defined over an appropriate distance metric in the measure space $\calp(\Real)$. We will follow the discretization approach defined earlier in~\eqref{eq: metric euler} in the study of gradient flows (see e.g.~\cite{jko98,ambrosio2006gradient}) and generate a sequence of solutions $\nu^{k}_h \in \left(\calp^2(\Real),  W_2\right)$ 
that each solves:
\begin{align}
    \nu_{h,i}^{k} = \arg\min_{\nu \in \calp(\R)} \left\{ V_i(\nu\,;\,\nu^{k-1}_h)\right\}, 
    \label{opt:vi-euler}
\end{align}
where $V_i(\nu\,;\,\nu^{k-1}_h)\;:=\;\frac 1 2 W_2(\nu^{k-1}_{h,i}, \nu)^2 + h\, J_i(\nu\,;\,\nu^{k-1}_{h,-i})$.


In the next section, we present a set of results that provide different representations of MFVI, including Sec.~\ref{sec:gf_MFVI} on gradient flow and Sec.~\ref{sec:pde} on PDE and diffusion processes. Each case presents a distinct view on the analysis of convergence of the iterates $\nu^{k}_{h,i}$. Note that the PDE representation we will analyse their convergence using the associated densities $\rho^{k}_{h,i}$, where the objectives defining optimization formulations~\eqref{eq:vi_system-peri} and~\eqref{opt:vi-euler} are represented as functionals of $\rho$. 

\section{MFVI: Gradient Flow Representation}
%
%
\label{sec:gf_MFVI}

A natural question arises regarding whether the sequence of iterates generated by solving \eqref{opt:vi-euler} contain a limit as the step size $h$ shrinks,
and if so, how to characterize the limit. These problems have been investigated extensively for the case in which the VB functional is studied on the Wasserstein space of probability measures; see e.g.~\cite{jko98, ambrosio2006gradient, trillos20alonso}, as well as their follow-ups. In essence, for this case, convergence is establish to a gradient flow under $J(\nu)$ as $h\searrow$.

In the case of MFVI, an analogous gradient flow can be defined on the product space of the component Wasserstein spaces. 
We will first establish the tightness of sequence of measures produced by \eqref{opt:vi-euler}, which via Prokhorov's Theorem~\cite{BillingsleyBook} guarantees the existence of the convergence subsequences. 
%
%
%
Analytically, tightness of the family of probability measures under consideration is equivalent to establishing that the family is sequentially compact. 
\begin{definition}[{\bf Tightness}]
A family $\calm$ of probability measures  on $X$ is tight if $\forall \epsilon >0$, $\exists$ compact set $K_\epsilon$ s.t. $\forall \mu\in \calm$, $\mu(X/K_\epsilon)< \epsilon$.  
\end{definition}

The next lemma establishes that algorithm \eqref{opt:vi-euler} produces a tight sequence of measures. Its proof can be found in Sec.~\ref{sec:tightness_proof}.

\begin{lemma}
\label{lem:tightness}
For each fixed step size $h>0$, the family of probability measures $\{\nu_{h,i}^k\}$, or equivalently  their densities $\{\rho_{h,i}^k \}$ are tight.
\end{lemma}
%

So, the sequences contain convergent subsequences. To show that they converge to a common limit, we utilize the uniqueness of gradient flow on product spaces defined by $\g$-convex functionals. For that purpose, we will first introduce some necessary basic concepts such as AC space, tangent, cotangent and the Fr\'echet subdifferential. 

\subsection{Gradient Flows on Product Wasserstein Space}
\label{sec:gradient_flow}

In this section, we will present the concepts and basic properties of gradient flows on a product Wasserstein space, the limiting gradient flows suited for MFVI.
We would like to point out that these concepts and basic properties presented here, as well as preparatory materials on some basic notions presented in Appendix \ref{sec:gradient_flow_appendix}, are natural extensions of those for gradient flows on a Wasserstein space, and can be found in various chapters in~\cite{ambrosio2006gradient}. For simplicity, we assume that each marginal is a measure in space $\calp_2(\Real)$ endowed with the Wasserstein metric $W_2$.
\begin{definition}[{\bf $AC_p$, spaces of absolutely continuous probability measures}]
For real values $a<b$, the space $AC_p(a,b; \prod_{i=1}^d \calp_i^2(\Real))$ consists of maps $v(s)=(v_i(s))_{i=1}^d$, where each component $v_i(s) : (a,b) \ra \calp^2(\Real)$ such that there exists a $\Real^d$-valued $L^p$ function $m(s)=(m_i(s))_{i=1}^d$ satisfying, for each $i=1,\ldots,d$:
\begin{align*}
W_2(v_i(s), v_i(t))&\le \int_s^t m_i(r) dr \;\;\forall a<s\le t<b.
\end{align*}
\end{definition}
For any real number $p>1$, define the product tangent bundle as follows.
\begin{definition}[{\bf Tangent bundle}]
For each $\nu=(\nu_1, \ldots, \nu_d) \in \prod_{i=1}^d \calp^2_i(\Real)$, define, 
\begin{align*}
T_\nu := \prod_{i=1}^d CL_p(\{ (j^i_q(\nabla \phi_i)): \phi=(\phi_1, \ldots, \phi_d) \in C_b^\infty(\Real)\})
\end{align*}
with $C_b^\infty(\Real)$ denotes the space of bounded smooth function on $\Real$, $j^i_q$ denotes the duality map on $L_p(\nu_i)$ ($f\mapsto f^{p/q}$), and $CL_p$ denotes the closure under $L_p$.
\end{definition}
The classic Fr\'echet subdifferential of a functional $\phi$ defined on a Banach space $\calb$, is defined as a subset in the dual space $\calb'$, more specifically, for an element $v\in D(\phi)$, 
\begin{align*}
\xi \in \partial \phi(v) \iff \lim \inf_{w \rightarrow v}
\frac{\phi(w)-[\phi(v)+\langle \xi, w-v\rangle]}{||w-v||_\calb} \ge 0,
\end{align*}
where $\langle \cdot, \cdot \rangle$ represents the duality action, which is reduced to the inner product when $\calb$ is a Hilbert space.

For $\calb$ being $\prod_{i=1}^n \calp_i^2(\Real)$, for the Fr\'echet differential at $\mu$, the dual space is the space of $L_2$ functions with $\mu$, and the displacement $w-v$ will be replaced an optimal transport plan between the two, as introduced in Ch. 10 of~\cite{ambrosio2006gradient}. To distinguish it from the classic Fr\'echet subdifferential, we name it Fr\'echet-Wasserstein subdifferential, and denoted as $\partial \phi(\mu)$. For any $\xi\in L^2(\mu)$, $\xi \in \partial \phi(\mu)$ if 
\begin{align*}
\lim \inf_{\nu \rightarrow \mu}
\frac{\phi(\nu)-[\phi(\mu)+\int_{\Real^d}\langle \xi(x), t^\nu_\mu(x)-x\rangle d \mu(x)]}{W_2(\mu,\nu)} \ge 0,
\end{align*}
where $t^\nu_\mu$ represents the transportation plan from $\mu$ to $\nu$ that solves the minimization defining $W_2(\mu,\nu)$. For the product space, we naturally consider the set of Fr\'echet-Wasserstein subdifferentials for each $i=1, \ldots, d$.
\begin{definition}[{\bf Gradient flow}]
A map $\mu(t) \in AC_p(a,b;\prod_{i=1}^d \calp_i^2(\Real))$ is a solution to the gradient flow equation
\begin{align}
\label{eqn:gradient_glow_defn}
j_p(v(t)) \in -\partial \phi(\mu(t)),
\end{align}
if for $v(t)\in T_{\mu(t)}$, its dual vector field $j_p(v(t))$ belongs to the subdifferential of $\phi$ at $\mu_t$. 
\end{definition}
Note that a gradient flow defined here takes the form of $\mu(t)=(\mu_1(t), \mu_2(t), \ldots, \mu_d(t))$, where each $\mu_i(t), i=1, \ldots, d\,$ can be viewed as a (marginal) gradient flow in the conventional sense. Meanwhile, the functional $\phi(\mu(t))$ also takes the form $(\,\phi_1(\mu(t)), \, \phi_2(\mu(t)),\ldots, \phi_d(\mu(t))\,)$. While \eqref{eqn:gradient_glow_defn} provides a more abstract geometric-in-nature definition, it can be understood as $\{\mu_i\}_{i=1}^d$ satisfying a system of energy dissipation equations similar to \eqref{eqn:Energy_Dissipation}: 
\begin{align*}
\phi_i(\mu(0)) =& \phi_i(\mu(t))  + \frac12 \int_0^t |{\dot \mu_i}((r)|^2 dr + \frac12 \int_0^t |\nabla \phi_i(\mu(t)) |^2 dr,
\end{align*}
with the derivatives similarly defined as in \eqref{eqn:Energy_Dissipation}. 

Next, we need to ensure the uniqueness of the gradient flow, that is, given a measure as the initial condition, there is a unique gradient flow for a given functional.
The key to uniqueness, derived from the contraction properties of the functional defined on the Wasserstein space, is its convexity properties. Note that though a functional may be linear on $\nu$, such as the expected log likelihood $\Psi$, the space $\calp$ of probability measures is itself not linear, that is, a linear combination of two probability measures is not a measure in general. We will exploit the notion of $\g$-convexity along geodesics in this space, and to provide precise definitions we need the following additional setup. 

First, 
given an underlying metric space $X$, 
for any two measures $\mu^1, \mu^2\in \calp^2(X)$, a curve $\mu_t\in \calp^2(\Real), t\in[0,1]$ that connects them ($\mu_0=\mu^1$ and $\mu_1=\mu^2$) is called a (constant speed) geodesic if 
$W_2(\mu_s,\mu_t) =(t-s) W_2(\mu^1,\mu^2)$ 
holds for all $0\le s\le t\le 1$. Denote by $\bmu \in \calp^2(X\times X)$ a joint probability measure defined on the product space of $X$, and let $\pi^i:X\times X \ra X$ be the projection to the $i$-th marginal, $i =1,2$. Define the push forward operator $\sharp$ as $\bmr_{\sharp}\bmu (A) = \bmu ( \bmr^{-1} ( A) )$ for any $A\subseteq X$. Finally, we can define the convex interpolation measure as $\mu_t^{i\ra j}:= (\pi^{i\ra j}_t)_\sharp \bmu$, where $\pi^{i\ra j}_t : = (1-t)\pi^i + t\pi^j$ and $t\in[0,1]$.  
\begin{definition}[{\bf $\g$-convexity along geodesic}]\label{defn:convexity}
Given a separable Hilbert space $X$ and $\phi: \calp^2(X) \ra (-\infty, +\infty]$, a value $\g \in \Real$, we say that $\phi$ is $\g$-geodesically convex in $\calp^2(X)$ if for every couple $\mu^1, \mu^2 \in \calp^2(X)$, there exists an optimal transfer plan $\bmu \in \Gamma(\mu^1, \mu^2) \subset \calp^2(X\times X)$ of joint measures with marginals $\mu^1$ and $\mu^2$ such that $ \forall t \in[0,1]$
\begin{align}
\phi(\mu^{1\ra 2}_t) \le &(1-t) \phi(\mu^1) + t \phi(\mu^2) 
-\frac{\g}{2} t(1-t) W^2_p(\mu^1, \mu^2).\label{eqn:geodesic_convex_defn}
\end{align}
\end{definition}

\begin{definition}[{\bf generalized geodesic}]
A generalized geodesic joining $\mu^2$ and $\mu^3$ (with base $\mu^1$) is the curve defined by $\mu_t^{2\ra 3}=(\pi^{2\ra 3})_\sharp \bmu$, with $\bmu \in \Gamma(\mu^1, \mu^2, \mu^3)$ and $\pi^{1,2}_\sharp \bmu \in \Gamma(\mu^1, \mu^2)$ and $\pi^{1,3}_\sharp \bmu \in \Gamma(\mu^1, \mu^3)$.
\end{definition}

\begin{definition}[{\bf $\g$-convexity along generalized geodesic}]
Given $X$, a separable Hilbert space and $\phi: \calp_p(x) \ra (-\infty, +\infty]$, $\g \in \Real$, we say that $\phi$ is $\g$-convexity along generalized geodesic if for any $\mu^1, \mu^2, \mu^3\in D(\phi)$, the domain of $\phi$, such that there is a generalized geodesic $\mu_t^{2\ra 3}$ such that, $\forall t \in[0,1]$,
\begin{align}
\phi(\mu_t^{2\ra 3}) \le & (1-t) \phi(\mu^2) + t\phi(\mu^3) 
-\frac{\g}{2} t(1-t) W_\mu^2(\mu^2, \mu^3). \label{eqn:generalized_geodesic_convex_defn}
\end{align}
\end{definition}
The concept of $\g$-convexity of a functional on a measure space along a geodesic closes matches that of $\g$-convexity of functions defined on a metric spaces.
\begin{definition}[{\bf $\g$-convexity}]\label{defn:l-convex-fn}
A function $f:\Real^d \ra \Real$ is called $\g$-convex for some $\g\ge 0$ if the following holds for any $x_1, x_2\in \Real^d$ and $t\in[0,1]$, 
\begin{align*}
f((1-t)x_1+t x_2) & \;\le\; (1-t)f(x_1) +tf(x_2) -\frac{\g}{2}t(1-t)|x_1-x_2|^2.
\end{align*}
\end{definition}
Indeed, the $\g$-convexity along geodesics of a integral functional $\calv(\mu)=\int_XV(x)d \mu$ is closely tied to the $\g$-convexity of the integrand $V$: 
\begin{proposition}
\label{pro:convexity}
If $V$ is a $\g$-convex function for some $\g \in \Real$, and its negative part has a $2$-growth, i.e.
\begin{align*}
V(x) \ge -A -B|x|^2, \quad x\in X, \hbox{for some $A, B\in \Real$.}
\end{align*}
then $\calv$ is $\g$-convex along generalized geodesic. 
\end{proposition}
This follows the first result of Proposition 9.3.2 in~\cite{ambrosio2006gradient}.

\begin{proposition}
\label{pro:contraction}
If the defining functional $\phi(\mu)$ is $\g$-convex along generalized geodesic for some $\g>0$, then the gradient flow $\mu(t)$ is a contraction, and it is unique given an initial state $\mu(0)$. 
\end{proposition}
\begin{proof}
The proof follows from Theorem 11.1.4  in~\cite{ambrosio2006gradient}. 
\end{proof}
\subsection{Convergence to Gradient Flow on Product Space }

Definitions and basic concepts developed in Sec. \ref{sec:gradient_flow}, as well as Appendix \ref{sec:gradient_flow_appendix} enables us to identify the limiting process of the outputs from the discrete algorithm \eqref{opt:vi-euler} as the step size goes to zero. To be more specific, the limiting gradient flow will take the form of $(\nu_1(t), \nu_2(t), \ldots, \nu_d(t))$ where each $\nu_i(t), i=1,\ldots, d$ represents the gradient flow in the conventional sense.

\begin{theorem}
\label{pro:gradientflow}
Suppose that the negative log-likelihood function ($-\log \pr (x,\theta)$) is $\g$-convex for some $\g>0$. Define the family of interpolated probability measures $(\nu_{h,1}(t), \nu_{h,2}(t),\ldots,\nu_{h,d}(t))$ for each $k \geq 1$ and  $i=1,\ldots, d$ as
$$\nu_{h,i}(t) = \nu^k_{h,i} \;\;\; \text{ for } t \in [kh, (k+1)h) \;\; $$ 
where $\{\nu^k_{h,i}\}$ are the updates generated by the discrete algorithm \eqref{opt:vi-euler}. Then, there exists $(\nu_i(t))_{i=1}^d$, a gradient flow on the product space $(\prod_{i=1}^d \calp_i^2(\Real),W_2)$ defined by a functional $\phi(\nu)=(\phi_1(\nu), \phi_2(\nu),\ldots, \phi_d(\nu))$ with $\phi_i(\nu)= J_i(\nu_i;\nu_{-i})$ for each $i=1,2,\ldots, d$, such that,  as $h\downarrow 0$, 
$$ \nu_{h,i}(t) \rightharpoonup \nu_i(t) \;\;\; \text{ weakly in $W_2$} \;\;\; \text{ for every } t\in (0,\infty).$$ 
\end{theorem}
\begin{proof}
Lemma \ref{lem:tightness} established the tightness (compactness) of the sequence of measures $\nu^k_{h,i}$, and thus guarantees the existence of convergent subsequences, and in turn limiting points. Straightforward calculations of the subdifferential verify that the relationship in the cotangent bundle \eqref{eqn:gradient_glow_defn} holds for the limiting points. The $\g$-convexity of the the negative log-likelihood function leads to the geodesic-$\g$ convexity of the functional due to the integration form of the functional by Proposition \ref{pro:convexity}.
This in turn implies that the gradient flow is a contraction, and hence the uniqueness of its solution by Proposition \ref{pro:contraction}. Therefore, all the convergent subsequences will converge to the same limiting point, thus the convergence of the sequence. 
\end{proof}




\section{MFVI: As Quasilinear PDE and SDE}
\label{sec:pde}

In this section, we will identify a quasilinear parabolic partial differential equation, and demonstrate in Corollary \ref{cor:density_convergence} that evolution of the density functions of the minimizing sequence for MFVI, as the step size tends to zero, converge weakly to the solution to these equations. Again, the uniqueness of the gradient flow obtained in the previous section plays an important role in establishing the limit by ensuring the
uniqueness of the solution to the PDE. Furthermore, this uniqueness property can be extended to cover a general class of PDE that possesses the necessary convexity structure in their formulation. 
This is a new result for this class of PDEs to the best of our knowledge. Finally, the form of the differential equations allow us to identify the stochastic differential equation whose weak solution has a density that coincides with the solution to the PDE, thus completes the picture of different representation of MFVI, which are illustrated in Figure \ref{fig:perspective_MFVI}.

Theorem~\ref{pro:gradientflow} addresses the convergence of the  sequence of MFVI solutions from~\eqref{opt:vi-euler} as elements in a product Wasserstein space of probability measures. When the probability measures involved are all absolutely continuous w.r.t. some reference measure, in our case the Lebesgue measure, Theorem~\ref{pro:gradientflow} leads to the following corollary on their densities (also known as Radon–Nikodym derivative, likelihood in various literature), which are shown to converge to a solution of a quasi-linear evolutional equation. 
\begin{cor}
\label{cor:density_convergence}
   Suppose that the negative log-likelihood function $-\log \pr (\pmb{x},\theta)$ 
   is $\g$-convex for some $\g>0$. For $i=1,\dots, d$, let $\big\{(\rho^k_{h,i})_{i=1}^d\big\}_{k\geq1}\in L^2$ be the densities associated with the measures produced by the iterative scheme \eqref{opt:vi-euler}, and let $\rho_{h,i}(t)$ be their interpolation of $t\in[0,\infty)$ for each $h,i$. 
    Then, as $h\downarrow 0$, 
    \[ 
    (\rho_{h,1}(t),\rho_{h,2}(t),\ldots, \rho_{h,d}(t))\rightharpoonup (\rho_1(t), \rho_2(t), \dots, \rho_d(t))\;\;,
    \] 
    weakly in $L^2(\R^d)$ for a. e.  $t\in (0,\infty)$, and $((\rho_1(t), \rho_2(t), \dots, \rho_d(t))\in  C^\infty((0,\infty)\times \R^d)$ is the unique solution of the following equation in its coordinate form, 
    \begin{align}
    \label{eqn:coupled_FP}
    \partial_t \rho_i = \partial_i (\partial_i \Psi_i(x, \rho_{-i}) \rho_i) + \partial^2_i \rho_i, \forall i=1,\ldots,d
    \end{align}
    with 
    proper initial conditions (see Sec. \ref{sec:proof}).
\end{cor}
\begin{proof}
The tightness of the probability measures confirmed in Lemma \ref{lem:tightness} is equivalent to the compactness in $L^2$ of the sequence of densities. Meanwhile, it is easy to see that any solution of \eqref{eqn:coupled_FP} can lead to a gradient flow on product space $(\prod_{i=1}^n \calp^2_i(\Real),W_2)$ defined by the functional $\phi(\nu)=(\phi_1(\nu), \phi_2(\nu),\ldots, \phi_d(\nu))$. Therefore, the uniqueness of the solution to \eqref{eqn:coupled_FP} follows from the uniqueness of the gradient flow. 
\end{proof}
The correspondence between the solution to an evolutional PDE and that of a gradient flow under the condition that the solution is smooth was first shown in~\cite{jko98}. More recent results under weaker conditions can be found in~\cite{a490137b659a404994231dd232758b74} and~\cite{BOGELEIN20143912}.

Corollary~\ref{cor:density_convergence} implies that the density of probability measure produced by MFVI algorithm is the solution to the homogeneous version of equation \eqref{eqn:coupled_FP}, i.e.
$$\partial_i  (\partial_i \Psi_i(x, \rho_{-i}) \rho_i) + \partial^2_i \rho_i =0 , \quad \forall i=1,\ldots,d.$$

\subsection{Uniqueness of quasilinear PDEs}

The equation \ref{eqn:coupled_FP} belongs to the following class of quasi-linear equation defined on $\Real^d\times \Real_+$,
\begin{equation} 
    \begin{cases}
    \partial_t u(x,t) = f(x,u,\nabla_x u) + \triangle_x u(x,t)
    \\
    u(x,0) = u^0(x) 
    \label{eq:quasilineaPDE}
    \end{cases}
\end{equation}
with necessary conditions for $f(x, \xi)$ to be made explicit later. 

\begin{definition}
A map $u(x,t) \in L^2([0,T);W^q_1(\Real^d))$ is called a weak solution to \eqref{eq:quasilineaPDE}, if for any $\phi(x,t)\in C^\infty_c([0,T]\times \Real^d)$, we have,
\begin{align}
\label{eqn:weak_soln-old}
&\int_{\Real^d} u\phi(T,x)dx - \int_{\Real^d} u\phi(0,x)dx
-\int_{\Real^d} u^0\partial_t \phi(t,x)\nonumber \\=&\int_{\Real^d} \nabla u \cdot \nabla \phi dx + \int_{\Real^d} f(x,u, \nabla u) \phi dx
\end{align}
\end{definition}
The existence of weak solutions in this form has been established in~\cite{GoudonSaad2001}. More specifically, the equation under consideration in~\cite{GoudonSaad2001} is in the following form,
\begin{align}
\label{eqn:quasi-linear-defn}
\left\{ \begin{array}{ll}
\partial_t u - \nabla \cdot (A(t,x)\nabla u)&\\
\phantom{\partial_t u } 
+B(t,x,u,\nabla u) =f,
&  \hbox{in $(0, T) \times \Omega$,} 
\\ u|_{t=0}= u_0, &  \hbox{in $\Omega$,} 
\\ u=0, & \hbox{on $(0, T) \times \partial \Omega$,}\end{array}\right.
\end{align}
where $\Omega$ is a regular open bounded set in $\R^N$. Moreover, $B$ has the following form,
\begin{align*}
B(t,x,u,\nabla u) =b(t,x) \cdot \nabla u + d(t,x)u + g(t,x,u, \nabla u).
\end{align*}
We will utilize these further assumptions.
\begin{asm} 
\label{asm:pde}\ 
\begin{itemize}
\item 
Boundedness of Coefficients: 
$A\in (L^\infty((0,T)\times \Omega))^{d\times d}$,  $b\in (L^\infty((0,T)\times\Omega))^d$ $d\in L^\infty((0,T)\times\Omega)$, $\nabla \cdot b \in L^\infty((0,T)\times\Omega)$;
\item
Uniform Ellipticity:
There exists $a>0$, such that, 
$A(t,x) \xi \cdot \xi \ge a|\xi|^2$;
\item The function $g: (0,T)\times\Omega \times \R \times \R^N \ra \R$ is measurable on $(0,T)\times\Omega$ for all $\g \in \R$, $\xi \in \R^d$, continuous with respect to $\g \in \R$, $\xi \in \R^d$, almost everywhere in $(0,T)\times\Omega$. Furthermore, $g$ satisfies both a sign condition and a growth condition with respect to the gradient variable since we suppose that,
$\g g(t,x,\g, \xi) \ge 0$,
and there exists $0\le \sigma <2 $ such that 
\begin{align}
|g(t,x,\g, \xi)| \le h(|\g|) (\gamma(t,x) + |\xi|^\sigma)
\end{align}
holds for all $\g \in \R, \xi\in \R^N$, and almost everywhere in $(0,T)\times\Omega$, with $\gamma \in L^2((0,T)\times\Omega)$, and $h$ a non decreasing function and $\zeta$-convex on $\R^+$ for some $\zeta>0$;
\item
$L^2$ data: $u_0\in L^2(\Omega), f\in L^2((0,T)\times\Omega)$. 
\end{itemize}
\end{asm}

Theorem 1 in~\cite{GoudonSaad2001} provides the existence of equation \eqref{eqn:quasi-linear-defn}, but not the uniqueness when the nonlinear term is not zero, to the best of our knowledge. The following result can be considered a contribution to the PDE literature. 
\begin{theorem}
\label{thm:pde_uniqueness}
Under Assumption \ref{asm:pde}, the solution to \eqref{eqn:quasi-linear-defn}
is unique.
\end{theorem}
\begin{proof}
The uniqueness of the solution to the PDE follows from the uniqueness of the solution to the gradient flow (in Hilbert space $W_2$ under the necessary convexity condition, which is guaranteed by the assumption on $g(t,x,\g,\xi)$.
\end{proof}


\subsection{MFVI: SDE Representation}

In ~\cite{doi:10.1137/17M1162780}, it is pointed out that solution to our quasilinear equation \eqref{eq:quasilineaPDE} can be viewed as the density function of a weak solution to the following stochastic differential equation
\begin{align*}
dX_t = \nabla \nabla \Psi(u(t,(X_t)) dt + dw_t,
\end{align*}
with $u(t,\cdot)$ denotes the density of $X(t)$ at time $t$, and $w_t$ a $d$-dimensional standard Brownian motion. This represents a Mckean-Vlasov process. For details on this type of processes, see, e.g.~\cite{Funaki1984}. The output of MFVI corresponds to the stationary distribution of this stochastic process, and thus completes the picture on representations of MFVI. This deep connection opens the doorway to SDE based solutions, a topic of future thorough investigations. 

\begin{figure}
	\centering
	\includegraphics[width=0.95\linewidth]
	{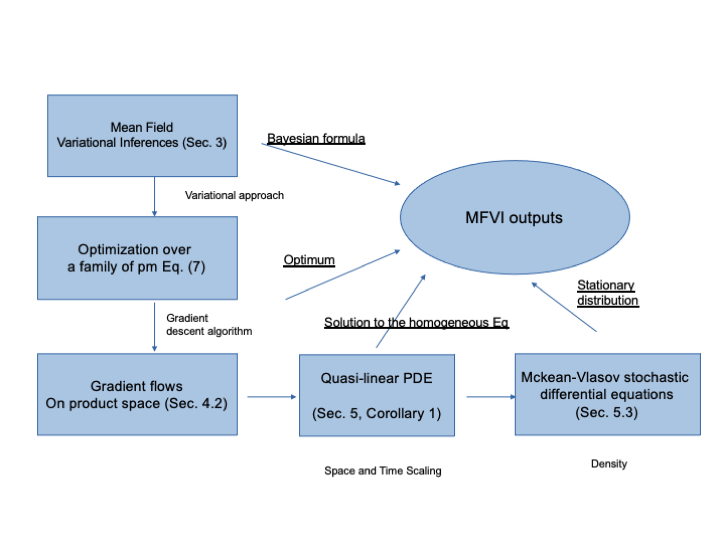}
	\caption{Representations of MFVI} 
	\label{fig:perspective_MFVI}
\end{figure}


\bibliographystyle{abbrv}
\bibliography{hmc}

\newpage

\onecolumn
\begin{appendix}

\section{Background Material} 

\subsection{Relationship between a diffusion process and FPE}
\label{sec:diffusion}

Definitions of diffusion processes can be found in many standard textbooks on stochastic processes and stochastic differential equations, and the following is taken from Sec. 2.5 of~\cite{arnold1974stochastic}, 
\begin{definition}
$X_t, t\in[t_0, T]$, a Markov process with almost certain continuous sample path is called a diffusion process if its transition probability $\pr(s,x,t, B)$ satisfies the following three conditions for all $s\in[t_0, T]$, $x\in\Real^d$ and $\epsilon>0$,
\begin{itemize}
\item
$\lim_{t\ra s}\frac{1}{t-s}\int_{|y-x|>\epsilon}\pr(s, x,t,dy) =0$;
\item
there exists an $\Real^d$-valued function $f(s,x)$ such that,
\begin{align*}
\lim_{t\ra s}\frac{1}{t-s}\int_{{\bar B}_\epsilon(x)}(y-x)\pr(s,x,t,dy) = f(s,x);
\end{align*}
\item
there exists an $\bbS^d$-valued function $B(s,x)$ such that,
\begin{align*}
\lim_{t\ra s}\frac{1}{t-s}\int_{{\bar B}_\epsilon(x)}(y-x)(y-x)^T\pr(s, x,t,dy) = B(s,x),
\end{align*}
\end{itemize}
with $\bbS^d$ denotes the set of $d\times d$ symmetric matrices, $B_\epsilon(x)=\{y \in \Real^d: |y-x|< \epsilon$, and ${\bar B}_\epsilon(x)$ its closure. 
The functions $f(s,x)$ and $B(s,x)$ are called the \emph{drift vector} and \emph{diffusion matrix} of $X_t$, respectively. 
\end{definition}
\begin{remark}
In the definition, the transition probability $\pr(s,x,t, B)$ refers to $\pr[X_t\in B|X_s=x]$, i.e. the probability of $X_t$ in a Borel set $B$ conditioning on $X_s=x$.
\end{remark}
The following classic result establishes the connection between the density function of a diffusion process with respect to the Lebesgue measure and solution to the Fokker-Planck equation. 

\begin{theorem}[Theorem 2.6.9. in~\cite{arnold1974stochastic}] If the derivatives $\frac{\partial}{\partial t} p, \nabla_x f(s,x)$ and $\nabla_x^2 B(s,x)$ exist and continuous, then for each $s\le t$, the density function $p(s,x,t,y)$ is a fundamental solution to the following Kolmogorov forward (Fokker-Planck) equation,
\begin{align}
&\frac{\partial}{\partial t} p+\sum_{i=1}^d \frac{\partial}{\partial y_i}(f_i(t,y) p) -\frac12 \sum_{i=1}^d\sum_{j=1}^d\frac{\partial^2}{\partial y_i\partial y_i}(B_{ij}(t,y) p)=0.\label{eqn:Kforward}
\end{align}
\end{theorem}
Equation \eqref{eqn:Kforward} is the Fokker-Planck equation in a general form, and naturally includes the equation corresponding to Bayesian inference, \eqref{eq:fpe}, as a special case. 
This connection in the theorem  facilitates the presentation of VB posterior as the stationary distribution (invariant measure) of the corresponding diffusion process. More formally, it is pointed out in~\cite[Proposition 2.1]{BiacaDogbe2017} that the following statements are equivalent:
\begin{itemize}
\item
$X_t$ is ergodic;
\item
$X_t$ has an invariant probability measure;
\item
There exists a unique probability measure $m$ such that, for any solution $u(x,t)$ to \eqref{eqn:Kforward} with $u_0$ as the initial value, we have, 
\begin{align*}
u(x,t) \ra \int u_0 dm, \quad \hbox{as } t\ra \infty,
\end{align*}
uniformly in $x$. 
\end{itemize}
It is also demonstrated (~\cite[Theorem 3.13]{BiacaDogbe2017}) that under an important assumption of the existence of (strong) Lyapunov function, the diffusion process admits a unique invariant measure. Moreover, the density of the measure is smooth if the coefficients are smooth. As an important example, when the drift term $f(t,x)$ takes the form of 
\begin{align*}
f(t,x) = -\nabla V(x),
\end{align*}
for some (potential) function $V(x)$, 
and $B(t,x)$ being a constant multiplier of the identity matrix, then the diffusion is known as the langevin diffusion, and the invariant measure, Gibbs measure.

\section{Background on Gradient Flows on the Product Wasserstein Space}
\label{sec:gradient_flow_appendix}

In Sec. 4, the product Wasserstein space has been defined specifically for probability measures on Euclidean space, with the purpose of avoiding unnecessary abstraction, as well as connecting more naturally with PDEs and diffusion processes. Here, those definition and concepts are more naturally provided for general metric spaces, with additional materials to form a coherent description of the product Wasserstein spaces.

\subsection{$AC_p$ spaces}

Suppose that, for each $i=1,\ldots, n$. $(X_i, d_i)$ is a (complete and separable) metric space, and $\calp(X_i)$ denotes the space formed by the (Borel) probability measures 
defined on $X_i$. As a metric space itself, $\calp(X_i)$ is equipped with the Wasserstein metric $W_p$. 
\begin{align*}
W_p^p(\mu_1, \mu_2)  :=\inf_\al \int_{X_i \times X_i} d_i(x,y)^p d\al(x,y),
\end{align*}
with the infimum taken over all joint measures $\al \in \calp(X_i \times X_i)$ with marginals $\mu_1$ and $\mu_2$ on the first and second factors.
The metric on $\prod_{i=1}^n \calp^2(X_i)$ is the natural product metric. 

\begin{definition}
The space of $AC_p(a,b;\prod_{i=1}^n \calp^2(X_i))$ consists of maps, $v(s)=(v_i(s))_{i=1}^n$, from $(a,b)$ to $\prod_{i=1}^n \calp^2(X_i)$ such that there exists a $\Real^n$-valued $L^p$ function $m(s)=(m_i(s))_{i=1}^n$ satisfying,
\begin{align*}
W_p(v_i(s), v_i(t))&\le \int_s^t m_i(r) dr, \quad \forall a<s\le t<b, &i=1,\ldots, n.
\end{align*}
\end{definition}
\begin{remark}
From optimal transport point of view, any two points in $\calp(X)$, which are two probability measures on a generic $X$, is connected by a "path", which corresponds to a solution to the continuity equation that produces an optimal transport plan. So these paths play the same role as line segments in finite dimensional Euclidean spaces. Consider just one point in $\calp_2(X)$ and its neighborhood, then the tangent space should be formed by these paths, see the formal definition below.
\end{remark}

\subsection{Product (co)-Tangent bundle}

For any real number $p>1$, the product tangent bundle is defined in  the following manner.
\begin{definition}
For each $\nu=(\nu_1, \ldots, \nu_d) \in \prod_{i=1}^n \calp_i(\Real)$, define, 
\begin{align*}
T_\nu := \prod_{i=1}^d CL_p(\{ (j^i_q(\nabla \phi_i)): \phi=(\phi_1, \ldots, \phi_d) \in C_b^\infty(\Real)\}
\end{align*}
with $j^i_q$ denotes the duality map on $L_p(\nu_i)$, and $CL_p$ denotes the closure under $L_p$.
\end{definition}
\begin{remark}
For $W_2$, it is proved in~\cite{10.2307/25662148} that the tangent cone at each point on the Wassertein space is a Hilbert space, and it is represented as the $L_2$ closure of Lipschitz functions. The definition given here is an extension of Definition 8.4.1 in~\cite{ambrosio2006gradient}, which is based on observations summarized in Theorem 8.3.1 in~\cite{ambrosio2006gradient}. More specifically, for a absolute curve $\mu_i$, corresponding vector field $v_i(t)$, $j_p(v_i)$ belongs to the closure of subspace generated by $\nabla \phi$ with $\phi\in CYL(X)$. $CYL(X)$ refers to the set of cylinder functions ( smooth function whose support lies in a finite dimensional subspace of $X$). Tangent bundle is understandably formed by those absolute continuous curves.

\end{remark}

\subsection{Fr\'echet-Wasserstein subdifferential}

The classic Fr\'echet subdifferential of a functional $\phi$ defined on a Banach space $\calb$, is defined as a subset in the dual space $\calb'$, more specifically, for an element $v\in D(\phi)$, 
\begin{align*}
\xi \in \partial \phi(v) \iff \lim \inf_{w \rightarrow v}
\frac{\phi(w)-[\phi(v)+\langle \xi, w-v\rangle]}{||w-v||_\calb} \ge 0,
\end{align*}
where $\langle \cdot, \cdot \rangle$ represents the duality action, which is reduced to the inner product when $\calb$ is a Hilbert space.

For $\calb$ being $(\prod_{i=1}^n \calp_i(\Real), W_p)$, for the Fr\'echet differential at $\mu$, the dual space is the space of $L_q$ functions with $\mu$, and the displacement $w-v$ will be replaced an optimal transport plan between the two, as introduced in Ch. 10 of~\cite{ambrosio2006gradient}. To distinguish it from the classic Fr\'echet subdifferential, we name it Fr\'echet-Wasserstein subdifferential, and denoted as $\partial \phi(\mu)$. For any $\xi\in L^p(\mu)$, $\xi \in \partial \phi(\mu)$ if 
\begin{align*}
\lim \inf_{\nu \rightarrow \mu}
\frac{\phi(\nu)-[\phi(\mu)+\int_{\Real^d}\langle \xi(x), t^\nu_\mu(x)-x\rangle d \mu(x)]}{W_p(\mu,\nu)} \ge 0,
\end{align*}
where $t^\nu_\mu$ represents the transportation plan from $\mu$ to $\nu$ that solves the minimization defining $W_p(\mu,\nu)$, and $\langle \cdot, \cdot \rangle$ represents $L_p-L_q$ duality action. For the product space, we naturally consider the set of Fr\'echet-Wasserstein subdifferentials for each $i=1, \ldots, d$.

\section{Proofs}
\label{sec:proof}

\subsection{Proof of Lemma \ref{lem:tightness}}
\label{sec:tightness_proof}
\begin{proof}
Recall that, $\rho^k=(\rho^k_{h,1}, \rho^k_{h,2}, \ldots, \rho^k_{h,d})$ is updated as follow, 
\begin{align*}
    \rho_{h,i}^k = \arg\min_{\rho_i \in P(\R)} \left\{ \frac12 W_2(\rho^{k-1}_i, \rho_i)^2 + h J_i^*(\rho_i;\rho_{-i})\right\}. 
\end{align*}
Here, we need to show that the sequence of the probability measures produced by the algorithm is tight (the probabilities of the complement of a compact set can be uniformly bounded). The key is to show that the second moments of this sequence of the probability measures can be uniformly bounded (then the tightness follows naturally from the Markov /Chebyshev inequality). While the argument in ~\cite{jko98} is rather technical, the essence of the proof is the convexity of the objective function. Because the second moment can be bounded by the cumulative square distance ($L_2$ or $W_2$) moving along the path, thus convexity means that this movement is monotone, and the distance can not be bigger than the distance between the initial position and the optimum. Although the space under consideration in of infinite dimension, since we move along a geodesic line, the problem is actually one dimensional.  

The following arguments utilize the cyclic coordinate update algorithm, described in Sec. \ref{sec:mfvi_formulation}, which is a common approach in convex optimization.   
For each $k\ge 1$, we have,
\begin{align*} 
    \frac12 W_2(\rho^{k-1}_{h,1},\rho_{h,1}^k)^2 + hJ_1(\rho_{h,1}^k;\rho_{h, -1}^{k-1}) \leq hJ_1(\rho_{h,1}^{k-1};\rho_{h, -1}^{k-1}),
\end{align*}
thus, 
\begin{align*} 
    \frac12 W_2(\rho^{k-1}_{h,1},\rho_{h,1}^k)^2 \leq hJ_1(\rho_{h,1}^{k-1};\rho_{h, -1}^{k-1})- hJ_1(\rho_{h,1}^k;\rho_{h, -1}^{k-1}).
\end{align*}
In the next step we will have 
\begin{align*} 
    \frac12 W_2(\rho^{k-1}_{h,2},\rho_{h,2}^k)^2 \leq hJ_2(\rho_{h,2}^{k-1};\rho_{h, -2}^{k-1})- hJ_2(\rho_{h,2}^k;\rho_{h, -2}^{k-1}).
\end{align*}
Meanwhile, it is also true that  
\begin{align*}
J_1(\rho_{h,1}^k;\rho_{h, -1}^{k-1})= J_2(\rho_{h,2}^{k-1};\rho_{h, -2}^{k-1}),
\end{align*}
because in the the cyclic coordinate update algorithm, the measure for $\theta_1$ in $\rho_{h, -2}^{k-1}$ has already been updated to $\rho_{h,1}^k$; recall that
\begin{align*}
J_1(\rho_{h,1}^k;\rho_{h, -1}^{k-1}) =& \int_\R \rho^k_{h,1} \left( \int_{\R^{d-1}} -\log p(x,\theta) \prod_{j=2}^d\rho_{h, j}^{k-1}\right) d\theta_1 \\ = & \int_{\R^{d}} -\log p(x,\theta) [\rho^k_{h,1} d\theta_1][\rho^{k-1}_{h,2} d\theta_2] \ldots [\rho^{k-1}_{h,d} d\theta_d].
\end{align*}
\begin{align*}
J_2(\rho_{h,2}^{k-1};\rho_{h, -2}^{k-1}) = & \int_\R \rho^{k-1}_{h,2} \left( \int_{\R^{d-1}} -\log p(x,\theta) [\rho^k_{h,1} d\theta_1]\prod_{j=3}^d\rho_{h, j}^{k-1}\right) d\theta_2 \\ =& \int_{\R^{d}} -\log p(x,\theta) [\rho^k_{h,1} d\theta_1][\rho^{k-1}_{h,2} d\theta_2] \ldots [\rho^{k-1}_{h,d} d\theta_d].
\end{align*}

Continuing with this process, we will have, 
\begin{align*}
\sum_{i=1}^d \frac12 W_2(\rho^{k-1}_{h,i},\rho_{h,i}^k)^2 \le hJ_1(\rho_{h,1}^{k-1};\rho_{h, -1}^{k-1})-hJ_d(\rho_{h,d}^{k};\rho_{h, -d}^{k-1}).
\end{align*}
Again, the cyclic coordinate update algorithm implies that,
\begin{align*}
J_d(\rho_{h,d}^{k};\rho_{h, -d}^{k-1})= J_1(\rho_{h,1}^{k};\rho_{h, -1}^{k}).
\end{align*}
This is seen from the forms:
\begin{align*}
  J_d(\rho_{h,d}^{k};\rho_{h, -d}^{k-1}) =& \int_\R \rho^k_{h,d} \left( \int_{\R^{d-1}} -\log p(x,\theta) \prod_{j=1}^{d-1}\rho_{h, j}^{k}\right) d\theta_d \\ = & \int_{\R^{d}} -\log p(x,\theta) [\rho^k_{h,1} d\theta_1][\rho^{k}_{h,2} d\theta_2] \ldots [\rho^{k}_{h,d} d\theta_d].
\end{align*}
\begin{align*}
J_1(\rho_{h,1}^{k};\rho_{h, -1}^{k})  =& \int_\R \rho^k_{h,1} \left( \int_{\R^{d-1}} -\log p(x,\theta) \prod_{j=2}^{d}\rho_{h, j}^{k}\right) d\theta_1 \\ = & \int_{\R^{d}} -\log p(x,\theta) [\rho^k_{h,1} d\theta_1][\rho^{k}_{h,2} d\theta_2] \ldots [\rho^{k}_{h,d} d\theta_d].
\end{align*}
Thus, telescoping the sums, we get that
\begin{align*}
\sum_{k=1}^N\sum_{i=1}^d \frac12 W_2(\rho^{k-1}_{h,i},\rho_{h,i}^k)^2 \le& hJ_1(\rho_{h,1}^{0};\rho_{h, -1}^{0})-hJ_d(\rho_{h,d}^{N-1};\rho_{h, -d}^{N-1}).
\end{align*}
Since the value of the functional $J$ is bounded, we can conclude that $\sum_{k=1}^N\sum_{i=1}^d \frac12 W_2(\rho^{k-1}_{h,i},\rho_{h,i}^k)^2 $ is uniformly bounded. This also implies that the second moments of the sequence of the measures can be bounded by this quantity, which gives us the desired tightness. 
\end{proof}

\subsection{Modified statement of Corollary 1 with initial value }
In the statement of Corollary 1, for the ease of exposition, the conditions with initial value have been ignored. Here, we provide a more complete statement of the corollary with those conditions included. 

\begin{cor}
\label{cor:density_convergence_complete}
   Suppose that the negative log-likelihood function $-\log \pr (\pmb{x},\theta)$ 
   is $\g$-convex for some $\g>0$. For $i=1,\dots, d$, let $\big\{(\rho^k_{h,i})_{i=1}^d\big\}_{k\geq1}\in L^2$ be the densities associated with the measures produced by the iterative scheme \eqref{opt:vi-euler}, and let $\rho_{h,i}(t)$ be their interpolation of $t\in[0,\infty)$ for each $h,i$. Let $\rho^0 = (\rho^0_1,\rho^0_2, \ldots, \rho^0_d) \in L^2(\R^d)$ such that for each $i=1,2,\ldots, d$, we have, $\rho^0_i(x)\ge 0$, $\int_\Real \rho^0_i(x) dx=1$,   $\int_{\R} x^2 \rho_i^0(x) dx < \infty$ and $J((\rho_1^0(x_1)dx_1, \ldots, \rho_d^0(x_d)dx_d)) < \infty$,
    Then, as $h\downarrow 0$, 
    \[ 
    (\rho_{h,1}(t),\rho_{h,2}(t),\ldots, \rho_{h,d}(t))\rightharpoonup (\rho_1(t), \rho_2(t), \dots, \rho_d(t))\;\;,
    \] 
    weakly in $L^2(\R^d)$ for a. e.  $t\in (0,\infty)$, and $((\rho_1(t), \rho_2(t), \dots, \rho_d(t))\in  C^\infty((0,\infty)\times \R^d)$ is the unique solution of the following equation in its coordinate form, 
    \begin{align*}
    \partial_t \rho_i = \partial_i (\partial_i \Psi_i(x, \rho_{-i}) \rho_i) + \partial^2_i \rho_i, \quad \forall i=1,\ldots,d
    \end{align*}
    with 
    initial conditions .
    \[
     \rho(t) \rightarrow \rho^0 \;\;\; \text{ converges strongly in } L^2(\R) \;\;\; \text{ as } t\downarrow 0.
     \]
\end{cor}
\end{appendix}

\end{document}